\documentclass[anonymous]{article}

\usepackage[preprint]{neurips_2023}

\usepackage[utf8]{inputenc} % allow utf-8 input
\usepackage[T1]{fontenc}    % use 8-bit T1 fonts
\usepackage[hidelinks]{hyperref}       % hyperlinks
\usepackage{url}            % simple URL typesetting
\usepackage{booktabs}       % professional-quality tables
\usepackage{amsfonts}       % blackboard math symbols
\usepackage{nicefrac}       % compact symbols for 1/2, etc.
\usepackage{microtype}      % microtypography
\usepackage{xcolor}         % colors

\usepackage{amsmath}
\usepackage{cleveref}
\usepackage{amssymb}
\usepackage{mathtools}
\usepackage{amsthm}
\usepackage{subfig}
\usepackage{graphicx}
\usepackage{bm}
\usepackage{slashed}
\usepackage{algorithm}
\usepackage{algorithmic}
\usepackage{multirow}
\usepackage{diagbox}
\usepackage{soul}
\usepackage[toc,page]{appendix}

\usepackage{enumitem}

\theoremstyle{plain}
\newtheorem{theorem}{Theorem}[section]
\newtheorem{proposition}[theorem]{Proposition}

\theoremstyle{definition}

\theoremstyle{remark}

\newcommand*{\ourtech}{SGN}

\newcommand*{\oursubteca}{NeurSpec}

\newcommand*{\oursubtecb}{NormSpec}

\newcommand{\ignore}[1]{}
\newcommand{\revise}[1]{{\color{black}#1}}

\title{SGN: A Similarity-based Generative Network for \\Data Generation under Distribution Shift}

\author{%
  Jiaqi Zhu\\
  National University of Singapore\\
  \texttt{jiaqi77@nus.edu.sg}\\
  \And
  Xincheng Chen\\
  National University of Singapore\\
  \texttt{e0838447@u.nus.edu.sg}\\
  \And
  Yuncheng Wu\\
  Renmin University of China\\
  \texttt{wuyuncheng@ruc.edu.cn}\\
  \And
  Zhaojing Luo\\
  Beijing Institute of Technology\\
  \texttt{zjluo@bit.edu.cn}\\
  \And
  Beng Chin Ooi\\
  Zhejiang University\\
  \texttt{ooibc@zju.edu.cn}\\
}

\begin{document}

\maketitle

\vspace{-2mm}
\begin{abstract}
Generative models trained on a source domain often produce samples that are poorly aligned with shifted target domains, limiting their effectiveness for target-domain data augmentation. Although target-specific adaptation can reduce this mismatch, it typically requires additional optimization and domain-specific parameters. We propose a Similarity-based Generative Network (SGN), a reusable framework that is trained once on labeled source data and applied to new target domains without parameter updates. SGN learns a latent space structured by label-induced pairwise similarities while preserving reconstructive information through an encoder-decoder architecture. At generation time, a small labeled representative set from the target domain is encoded and combined in the learned latent space, allowing the generated samples to inherit target-specific characteristics while maintaining class consistency. We further analyze the realizability and dimensionality requirements of the proposed similarity structure.
Experiments on image and tabular datasets demonstrate the effectiveness of SGN for target-guided data augmentation under source-to-target distribution shifts.
\end{abstract}

\ignore{
The core of training a generative network is to develop an evaluation mechanism of generated samples. 
Existing works adopt distribution-based evaluation mechanisms, which restrict the flexibility of the generated data distribution. In this paper, we propose a novel similarity-based evaluation mechanism, which utilizes data labels to construct similarity measurements. Based on this, we design a new generative 
\textcolor{black}{network}
%mechanism 
\ourtech{} 
%that supports 
to support
flexible distribution setups, 
%and thus, is able
and hence
%to
address the distribution shift challenge
%problem compared to
faced by
existing generative paradigms. 
We validate the effectiveness of \ourtech~through 
%an extensive experimental study over 
\revise{
extensive experiments on both image and tabular datasets.}
The results show that \ourtech~is capable of generating data in various distributions, enabling it to outperform %\revise{state-of-the-art baselines}
\textcolor{black}{existing conditional generative models (CGMs)} 
\textcolor{black}{in terms of either downstream classification tasks or generated data quality}.}

\vspace{-2mm}
\section{Introduction}
\vspace{-2mm}
Generative models (GMs) provide a principled way to expand limited datasets by synthesizing new samples that preserve the structure of observed data.
Representative paradigms, such as generative adversarial networks (GANs)~\citep{NIPS2014_5ca3e9b1}, variational autoencoders (VAEs)~\citep{DBLP:journals/corr/KingmaW13}, and more recently, diffusion-based models~\citep{DBLP:conf/nips/HoJA20,DBLP:journals/csur/YangZSHXZZCY24}, have attracted increasing attention for \revise{data generation} in various applications~\citep{deldjoo2021survey,zhu2026generative,yi2017dualgan,shaham2019singan}.
Despite their architectural differences, these methods are typically trained to approximate the data distribution observed during training. Conditional generative models further incorporate class labels, semantic attributes, or domain information to control the characteristics of generated samples~\citep{mirza2014conditional,miyato2018cgans,DBLP:conf/iccv/ChouBH23, zhu2025context}. When the training and deployment data follow similar distributions, such models can generate realistic and useful samples for data augmentation and downstream prediction.

However, in many practical settings, a generative model trained on one source domain is reused across multiple target domains with different data distributions. Such shifts may arise from changes in acquisition devices, environments, populations, data-collection protocols, or class proportions~\citep{quinonero2008dataset, zhu2023meter}.
As a result, samples generated from the source distribution may be poorly aligned with the target domain and provide limited, or even negative, utility for target-domain augmentation. Although target-specific fine-tuning can mitigate this mismatch, it is often costly and unreliable when only limited target data are available.

We therefore consider a reusable generation setting in which a model is trained once on a labeled source dataset and applied to multiple target domains without parameter updates, using only a small labeled representative set at generation time.
The key challenge is to transfer class semantics learned from the source domain while preserving target-specific characteristics.
% Conventional class-conditional models primarily learn source-domain conditional distributions through fixed label representations, whereas direct interpolation among target samples captures target characteristics but lacks an explicit semantic structure to constrain the generated data.
Existing approaches only partially address this challenge.
\textit{Source-trained conditional generators} use class labels or semantic attributes to control generation, but the learned class-conditional distributions remain largely tied to the source domain~\citep{miyato2018cgans,DBLP:conf/icml/HouCSPLC22,DBLP:conf/iccv/PeeblesX23,DBLP:conf/iclr/SadatKHW25}.
\textit{Target-specific adaptation methods} can align a pretrained generator with a new domain through fine-tuning or additional optimization, but require maintaining or updating a separate model for each target domain~\citep{DBLP:conf/cvpr/WangGBHK020,DBLP:journals/corr/abs-2511-18281}.
\textit{Sample-based augmentation methods}, such as input-space mixing and latent interpolation, can directly incorporate target-domain characteristics from limited observations, yet they typically lack an explicitly learned semantic structure to preserve class consistency during generation~\citep{DBLP:conf/iclr/ZhangCDL18,DBLP:conf/icml/VermaLBNMLB19,nguyen2026targeted}.
% source-trained conditional generators remain tied to source-domain distributions, target-specific adaptation requires additional optimization for each new domain, and sample-based interpolation captures target variation without an explicit semantic constraint.
These limitations motivate a reusable framework that combines transferable class structure with target-domain information during generation.

In this paper, we propose a novel evaluation mechanism that 
%creates 
%leads to
forms a basis for
an alternative paradigm for \textcolor{black}{GMs}. Specifically, we consider a generated sample to satisfy the evaluation criteria if its similarity scores with other data samples are close to the similarity scores of original data samples;
%%% ooibc: what are "those" original data samples
%%% those is ambiguous!
%%% xcc: Thanks for pointing out prof. "those of" can be changed to the "similarity scores of"
otherwise, the quality of the generated data is poor.
%undesirable. 
%%% ooibc:undersirable?
%%%  
%%% xcc: Thanks for pointing out, prof. I am not sure. Maybe poor is better?
The rationale behind this mechanism is that the original data points definitely satisfy the evaluation criteria; therefore, if a generated data point results in the same performance as the original data points, it can be considered a satisfactory sample. The key to this mechanism is to develop an appropriate similarity measurement, which is challenging. If the measurement is tight, the generated samples will only be accepted if they are extremely close to the original data points. On the contrary, if the measurement is loose, the generated samples may not be acceptable even if they satisfy the evaluation criteria. To address this challenge, we utilize data labels to construct the similarity measurement because data labels provide a natural balance between the generalization and the specification of measurement, where sufficient numbers of data are categorized into each label. We therefore design a zero-one similarity matrix: the similarity score of two samples is one if they have the same label, and zero otherwise. 
%%%jq: too many 'Specifically'

Based on the evaluation mechanism, we propose a new generative network, called similarity-based generative network (\ourtech{}). Specifically, our evaluation mechanism establishes a latent space where each data sample has a set of colinear latent representations and a unique special latent representation. 
%we assume that each data point has multiple co-linear latent representations. Notably, there exists a special latent representation for each data point, such that 
The similarity score between any two data points is reflected by the dot product of their special 
%%% ooibc:special?
%%% xcc: Yes, prof. Only a part of latent representations need to satisfy the zero-one similarity matrix.
%% who knows what is special?
%%% xcc: Thanks for pointing out, prof. I will revise it.
latent representations. 
Consequently, \revise{we train a neural network model}
%neural networks are trained 
to find the latent space and also the transformation between the latent and the original spaces. 
%%%jq: what does the 'neural network' here mean? SGN?
Essentially, \ourtech{} learns the special latent representations of the data points in the training dataset without being aware of its distribution. After \ourtech{} is trained, we utilize these latent representations to generate desired samples according to the need of downstream tasks. %\textcolor{blue}{/new parties}. 
%%% ooibc: below
Compared to existing \textcolor{black}{GM}s, \ourtech{} is 
more amenable to
%% ooibc: more amenable to
generating samples that follow various distributions and is capable of addressing the distribution shift problem, where the data distribution in the source domain is different from that in the target domain.
% \citep{quinonero2008dataset}.
%%% ooibc: this sentence is so hard to parse!
%%%jq: i.e.,... or (i.e., ...) need to be unified

Mathematically, \ourtech{} can be viewed as a conditional generative paradigm. However, unlike existing conditional generative works that utilize the label information explicitly and represent each label with an embedding, \ourtech{} implicitly adopts label information, which could better alleviate the distribution shift problem because it provides a stable label-data relationship. In particular, in our design, the inner product of any two data representations always satisfies the same zero-one similarity requirement across different domains. 
%While 
In existing works, a data sample
%one data
%%% ooibc: one data?
%%% how to measure?
is constantly represented as a one-hot label representation, while the data representations vary across different domains.%, the label representation remains the same (i.e., the one-hot vector); 
%%% ooibc: i.e
Thus, the relationship between the label and data representations exhibits significant variations across different domains.

Our contributions are summarized as follows:
\begin{itemize}[itemsep=0mm,leftmargin=4mm]
\vspace{-3mm}
    \item We propose a novel evaluation mechanism based on similarity measurements between different data points. To the best of our knowledge, this is the first work that designs a distribution-agnostic evaluation mechanism for \textcolor{black}{GM}s.
    \item We present a similarity-based generative 
    %scheme
    network
    \ourtech{} to generate desirable samples for %downstream tasks
    \textcolor{black}{target domains}, mitigating the distribution shift problem. We further analyze the lower bound of the latent space dimension and provide the convergence analysis of \ourtech{} training. 
    \item We conduct extensive experiments on \revise{eleven} real-world datasets including both \textcolor{black}{image and tabular} datasets, demonstrating the effectiveness of \ourtech{} for generating high-quality data and mitigating the distribution shift problem, \revise{compared to six state-of-the-art baselines}.
\end{itemize}

%%%jq: need to unify generative models and GMs

% \vspace{-3mm}
\section{Related Work}

%\subsection{Conditional Generative Model}
\textbf{Conditional Generative Model.}
Conditional generative models synthesize data according to auxiliary conditions, such as class labels, semantic attributes, text descriptions, or domain information.
Early representative approaches include conditional generative adversarial networks (CGANs)~\citep{mirza2014conditional}, conditional variational autoencoders (CVAEs)~\citep{sohn2015learning}, and projection-based conditional GANs~\citep{miyato2018cgans}. Subsequent studies improved class-conditional generation through contrastive objectives, auxiliary classifiers, and more expressive conditioning mechanisms~\citep{kang2020contragan,kang2021rebooting}.

More recently, diffusion models have become a prevalent paradigm for conditional generation. Classifier and classifier-free guidance enable pretrained diffusion models to generate samples aligned with class labels, textual descriptions, or other conditions~\citep{DBLP:conf/nips/DhariwalN21, DBLP:journals/corr/abs-2207-12598}. Recent studies further investigate the quality-diversity trade-off of guidance and develop more flexible conditioning mechanisms~\citep{DBLP:conf/iclr/SadatKHW25,DBLP:conf/iclr/LiHCLS0025}. Nevertheless, the conditional distributions learned by these models remain primarily shaped by their training data. When a generator trained on a source domain is directly applied to a shifted target domain, specifying the desired class does not by itself capture how the class-conditional data characteristics change across domains.

\textbf{Target-specific Generative Adaptation.}
Another line of research adapts a pretrained generator to a target domain using a limited number of target examples. Earlier GAN-based methods transfer source-domain generators through selective fine-tuning, regularization, or cross-domain correspondence preservation~\cite{DBLP:conf/cvpr/WangGBHK020,DBLP:journals/corr/abs-2511-18281}. Similar ideas have been extended to diffusion models, where a pretrained model is adapted using target-domain fine-tuning, domain-specific regularization, parameter-efficient updates, or additional guidance objectives~\cite{DBLP:journals/corr/abs-2211-03264}.
For example, Domain Guidance~\citep{DBLP:conf/iclr/ZhongZ0L25} interprets diffusion-model transfer as conditional generation and introduces a target-domain guidance mechanism, whereas SaRA~\citep{DBLP:conf/iclr/HuZYHWM25} performs parameter-efficient adaptation through sparse low-rank updates. These approaches can produce high-quality samples in a new domain, but still involve target-specific optimization and typically require storing either a separately adapted model or domain-specific parameters. In contrast, SGN keeps the source-trained model fixed and incorporates target-domain characteristics through representative samples only during generation.

\textbf{Sample-based Data Augmentation.}
Sample-based augmentation constructs new training examples directly from observed samples or their representations.
Classical methods include Mixup~\citep{DBLP:conf/iclr/ZhangCDL18}, which linearly combines training examples and labels, and Manifold Mixup~\citep{DBLP:conf/icml/VermaLBNMLB19}, which performs interpolation in hidden representation spaces.
% Autoencoder-based methods similarly generate samples by interpolating latent representations, providing a simple and efficient alternative to fitting a complete target-domain generative model.
Recent diffusion-based augmentation methods leverage strong pretrained generative priors to improve the fidelity and diversity of synthetic samples.
\cite{DBLP:conf/iclr/TrabuccoDGS24} edits real images using pretrained text-to-image models, while \cite{DBLP:conf/cvpr/Wang025a} performs class-wise interpolation between diffusion inversions and reconstructs the interpolated representations through a two-stage denoising process.
More recent work studies which synthetic samples are most useful rather than generating them indiscriminately. In particular, targeted image augmentation selectively allocates diffusion-based augmentation to training examples for which synthetic data are expected to provide greater utility~\citep{nguyen2026targeted}.

These approaches demonstrate the effectiveness of sample-conditioned and representation-space augmentation. However, conventional interpolation rules are generally defined by geometric proximity in the input or latent space.
Although modern pretrained representations may already encode semantic information, they do not explicitly impose a transferable pairwise class geometry learned from the source domain. SGN instead organizes the latent space using label-induced similarity constraints and uses this structure to guide the combination of limited target-domain representations without updating the model parameters.

\section{Similarity-Based Generative Networks}
\label{sec:sgn}

In this section, we present our similarity-based generative networks \ourtech{} that can generate desirable samples for \textcolor{black}{the need of target domains}. 
Generally speaking, \ourtech{} isolates the \textcolor{black}{source domain training} and \textcolor{black}{target domain generation}, where \ourtech~is first trained in the source domain without knowing the information of 
the target domain, and then \ourtech~generates required data samples of the target domain. 
Specifically, we first propose the similarity-based evaluation mechanism and then analyze the latent space it determines in Section~\ref{subsec:eval-mechanism}. 
Subsequently, we describe the training paradigm of \ourtech{} in Section~\ref{subsec:sgn-train}, and present the convergence analysis of the training process in Section~\ref{subsec:analysis}. 
Finally, we describe how to generate desirable samples in the target domain in Section~\ref{subsec:sgn-generation}.

We consider a labeled dataset $\mathbb{X} = \{(\bm{x}_i,y_{i}) | i \in \{1, \cdots, N\}\}$, where $N$ is the number of samples in $\mathbb{X}$. 
Let $K$ be the number of distinct labels and $\bm{a} = \{\bm{a}_1, \cdots, \bm{a}_K\}$ be the label value set. 
Moreover, we denote \textcolor{black}{the $i^{\text{th}}$ original data and its latent representations} as $\bm{x}_i^{\text{ori}}$ and $\bm{x}_i^{\text{lat}}$, respectively. 

%\vspace{-2mm}
\subsection{Similarity-Based Evaluation Mechanism}
\label{subsec:eval-mechanism}

We utilize the similarity score between samples to construct the evaluation mechanism, where the similarity scores are determined by data labels.
Specifically, we require the similarity score of two samples to be one if they have the same label, and otherwise zero. 
Let $\bm{S}$ denote the zero-one similarity matrix defined by the evaluation mechanism. 

In essence, this similarity-based evaluation mechanism determines a latent linear space. 
In this latent space, we can utilize the inner product of two data points' representations as their similarity score, where each data point has a special latent representation and 
%multiple\textcolor{blue}{/single} 
%%% zj: for the "special latent representation", it appears over 30 times in this section, can you give it a name instead of "special latent representation"?
%%% xcc: Thanks for the comment. I am not sure whether we can use the abbreviation of special latent representation, i.e., SLT.
%%% jq:  'each data point has a special latent representation and a set of other latent representations'  why? need further explanation
%%% xcc: A similar problem has been asked by zj. Please refer the answer to the problem: 'what is the advantage of "only the special latent representation needs to satisfy the zero-one similarity requirement"?'
%%% jq: Does the latent representation here refer to the output of the encoder?
%%% xcc: Yes
\textcolor{black}{a set of}
other latent representations. 
The inner product of the special latent representations is equivalent to the similarity score, and all latent representations for one data sample are co-linear. 
%%% (done) zj: inner product or cosine similarity? --> the latter is normalized
%%%xcc: inner product; cosine similarity can be considered a special case.
Formally, given a latent space with $D$ dimensions, $\bm{S}$ is encoded through
\begin{equation}
\label{equation15}
\bm{S}_{i,j} = \langle \bm{x}_i^{\text{lat\_S}}, \bm{x}_j^{\text{lat\_S}} \rangle =
\begin{cases}
1 & y_{i} = y_{j}\\
0 & y_{i} \neq y_{j},
\end{cases}
\end{equation}
where $\bm{x}_i^{\text{lat\_S}}$ and $\bm{x}_j^{\text{lat\_S}}$ are the special latent representations of data points $i$ and $j$, respectively, and $\langle \bm{x}_i^{\text{lat\_S}}, \bm{x}_j^{\text{lat\_S}} \rangle$ is the inner product of $\bm{x}_i^{\text{lat\_S}}$ and $\bm{x}_j^{\text{lat\_S}}$.

\begin{figure}[t]
	\centering
     \includegraphics[scale=0.4]{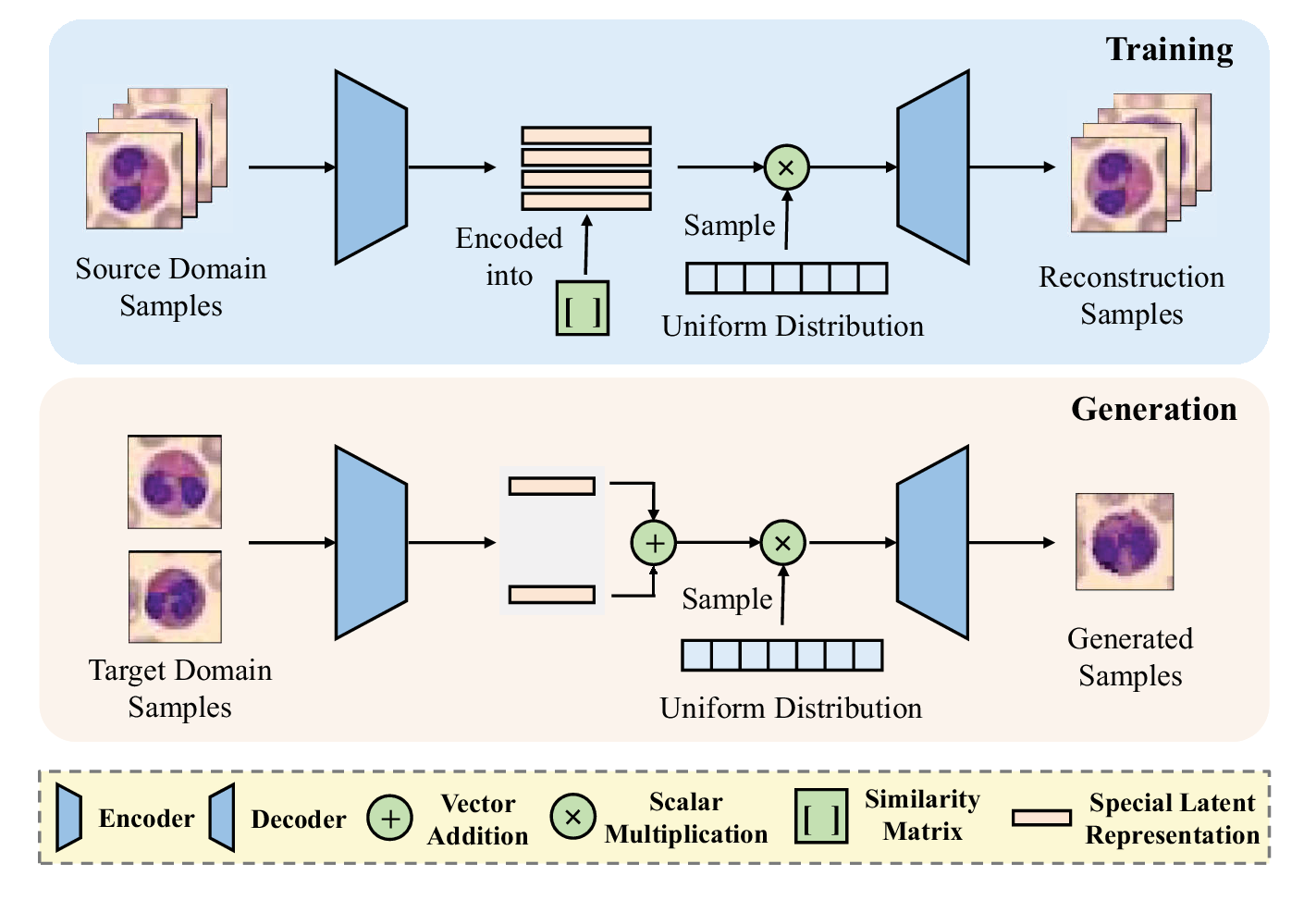}
     \vspace{-2mm}
	\caption{An illustration of the \ourtech~paradigm, which consists of two stages. 
    %\ourtech~is built on the basis of the similarity-based evaluation mechanism. 
    In the training stage, the encoder and decoder learn the transformation between the original data space and latent space, where the latent space satisfies the similarity requirement defined in Equation~\ref{equation15}. In the data generation stage, the encoder takes two data points with the same label and produces two latent representations; then, the decoder generates a data point by a linear combination of these latent representations as inputs.
    %a generated data point 
    %that satisfies the mechanism requirement 
    %is decoded from
    %a linear combination of the latent representations of data points with the same label.
    }
    \vspace{-3mm}
    \label{f1}
\end{figure}
%%%jq: please change the dataset's figures

The above design has \revise{two} properties. 
First, it ensures that \textcolor{black}{the latent representations of different labels} are orthogonal, i.e., \textcolor{black}{latent representations of all labels} essentially form an orthogonal set of the latent space, and each sample's latent representation is a linear combination of its label's corresponding orthogonal basis. 
%%% zj: "each sample's latent representation is a linear combination of its label's corresponding orthogonal basis" --> this needs more explanation
%%%xcc: I am not sure whether I understand the problem properly. The reason of "each sample's latent representation is a linear combination of its label's corresponding orthogonal basis" is "all labels' latent representations essentially form an orthogonal set of the latent space".
%%% zj: do you mean the basis of different labels are orthogonal, and each sample's latent representation is a linear combination of its label's "corresponding basis" --> the basis of the same label is not orthogonal
%%% xcc: Yes. 
Let $\bm{V} = \{\bm{V}_d | d \in \{1,...,D\}\}$ denote the aforementioned orthogonal set of the latent space, where different labels have different $\bm{V}_d$ and each label may have multiple orthogonal basis. Let $m_k$ be the number of orthogonal basis for the $k^{\text{th}}$ label ${\bm{a}}_k$, then $D = \sum_{k=1}^{K} m_k$.
%\begin{equation}
%\label{equation1}
%D = \sum_{k=1}^{K} m_k.
%\end{equation}

We use an indicator $\gamma_{d,k}$ to represent the relationship between $\bm{V}_d$ and ${\bm{a}}_k$: if $\bm{V}_d$ is an orthogonal basis of ${\bm{a}}_k$, then $\gamma_{d,k} = 1$; otherwise $\gamma_{d,k} = 0$. As a consequence, $\bm{x}_{i}^{\text{lat}}$ could be represented by $\bm{V}$ as
\begin{equation}
\label{equation3}
\bm{x}_{i}^{\text{lat}} = \sum_{d=1}^{D}\sum_{k=1}^{K}\beta_{i,d}\alpha_{i,k}\gamma_{d,k}\bm{V}_d, 
\end{equation}
where $\beta_{i,d}$ is the linear coefficient, and $\alpha_{i,k}$ indicates whether $y_i$ equals to the label value ${\bm{a}}_k$, i.e., if ${\bm{a}}_k = y_i$, $\alpha_{i,k} = 1$; otherwise, $\alpha_{i,k} = 0$. 

Second, the co-linear setting supports a flexible generative paradigm design, enabling the special latent representation to be set in a variety of ways. 
For example, we could use the normalized representation as the special latent representation due to the co-linear setting. 
Another option is to adopt a neural network to infer the special latent representation, where the network's input is the original representation and the output can be set to the special latent representation. 

\subsection{\ourtech{} Training}
\label{subsec:sgn-train}

Notably, \ourtech~refers to a family of deep neural networks trained based on the similarity-based evaluation mechanism. 
It should satisfy two requirements: 
(1) find the latent space determined by the similarity-based evaluation mechanism, i.e., the inner product of any two data points' special latent representations is equivalent to the similarity of these two data points; 
(2) imitate the transformation between the latent space and the original space, i.e., given a latent representation of a data point, \ourtech~could obtain its original representation. Therefore, the loss function is defined as follows:
\begin{equation}
\label{equation6}
\mathcal{L} =\frac{\eta}{N^2}\sum_{i=1}^{N}\sum_{j=1}^{N}\mathcal{L}_{i,j}^{(\text{r1})} + \frac{1 - \eta}{N}\sum_{i=1}^{N}\mathcal{L}_{i}^{(\text{r2})}, 
\end{equation}
%% zj: please also justify the design of coefficients: \frac{\eta}{N^2} and \frac{1 - \eta}{N}\sum_{i=1}^{N}, these two coefficients does not sum to one
%%% xcc: Here N^2 and N are the number of \mathcal{L}_{i,j}^{(\text{r1})} and \mathcal{L}_{i}^{(\text{r2})}, respectively. If we consider \frac{1}{N^2}\sum_{i=1}^{N}\sum_{j=1}^{N}\mathcal{L}_{i,j}^{(\text{r1})} as one loss and recognize \frac{1}{N}\sum_{i=1}^{N}\mathcal{L}_{i}^{(\text{r2})} as another loss. Then, \eta and 1 - \eta are their coefficients, and the sum of coefficients is 1.
where $\mathcal{L}_{i,j}^{(\text{r1})}$ corresponds to the regularization loss, which is the first requirement for any two data points, and $\mathcal{L}_{i}^{(\text{r2})}$ corresponds to the reconstruction loss, which is the second requirement for any specific training data point. $\eta \in [0,1]$ is a \textcolor{black}{trade-off parameter} which represents the relative importance between two requirements. 
\textcolor{black}{
%Note that there are multiple ways to define $\mathcal{L}_{i,j}^{(\text{r1})}$ and $\mathcal{L}_{i}^{(\text{r2})}$, i.e., \ourtech~ is general and can represent a family of realizations. In the following, we introduce two designs. 
Figure~\ref{f1} illustrates the general pipeline of \ourtech, which follows the autoencoder structure: $\bm{x}_{i}^{\text{ori}}$ and $\bm{x}_i^{\text{lat}}$ are the input and output of the encoder, respectively.
%In one-multiple setting, 
The input of the decoder is a scalar multiplication of $\bm{x}_i^{\text{lat}}$ that supports co-linear relationships among a data point's latent representations, i.e., $c\frac{\bm{x}_i^{\text{lat}}}{\|\bm{x}_i^{\text{lat}}\|}$, 
%%% (done) zj: do you mean unit vector?
%%% xcc: Yes, it is related to unit vector. The input of the decoder is a random variable multiplying a unit vector.
where $c$ is sampled from the uniform distribution $U(\|\bm{x}_i^{\text{lat}}\| - C,\|\bm{x}_i^{\text{lat}}\| + C)$. 
The constant $C$ is the maximum length variation of the latent representation.
%%% (done) zj: how is the constant $C$ decided? calculated or by experiment? 
%%% xcc: $C$ is decided by the experiment.
%In one-one setting, the input of the decoder is just $\bm{x}_i^{\text{lat\_S}}$; 
Let $\bm{x}_{i}^{\text{rec}}$ denote the output of the decoder.}
\revise{Then the reconstruction loss $\mathcal{L}_{i}^{(\text{r2})}$ can be defined as
$\mathcal{L}_{i}^{(\text{r2})} = \|\bm{x}_{i}^{\text{rec}} - \bm{x}_{i}^{\text{ori}}\|$.}
%In fact, the decoder's output is expected to be the same as the encoder's input $\bm{x}_{i}^{\text{ori}}$ (as assumed in the autoencoder framework), but there may exist losses. Hence, we use different notations to denote the input and the output of the autoencoder and use $\mathcal{L}_{i}^{(\text{r2})} = \|\bm{x}_{i}^{\text{rec}} - \bm{x}_{i}^{\text{ori}}\|$ to measure the differences between the original and the reconstructed representations.
%
\revise{
For $\mathcal{L}_{i,j}^{(\text{r1})}$, we define the special latent representation as follows:
\begin{equation}
\label{equation70}
\mathcal{L}_{i,j}^{(\text{r1})} =
\begin{cases}
\left|\langle f(\bm{x}_i^{\text{lat}}), f(\bm{x}_j^{\text{lat}}) \rangle\right| & y_{i} \neq y_{j}\\
\left|\langle f(\bm{x}_i^{\text{lat}}), f(\bm{x}_j^{\text{lat}}) \rangle - 1\right|  & y_{i} = y_{j},
\end{cases}
\end{equation}
where $f(\cdot)$ is a post-processing function of the encoder's output and $\bm{x}_i^{\text{lat\_S}} = f(\bm{x}_i^{\text{lat}})$. For example, it can be the identity function $f(\bm{x})$ or the normalized function $f(\bm{x}) = \frac{\bm{x}}{\|\bm{x}\|}$. We shall discuss the choice of this function for different types of datasets in Section~\ref{sec:experiment}.
}

\textbf{Determining the dimension of the special latent representation.} Since the number of samples $N$ is usually large, we utilize a batch of samples to calculate $\mathcal{L}$ in each epoch.
Let $D$ be the dimension of the latent space. 
The batch size can be used to derive the lower bound of $D$.
Let $\mathbb{X}_{\text{b}}$ be a batch of the training data and $B$ be the batch size. 
Suppose $\mathbb{X}_{\text{b}}$ has $Q$ different label values ($Q \le K$), and $N_{q}^{(\text{b})}$ is the number of data points with label value $\bm{a}_q$ in $\mathbb{X}_{\text{b}}$, and apparently $\sum_{q = 1}^{Q} N_{q}^{(\text{b})} = B$.
Moreover, to facilitate our analysis, we assume data in $\mathbb{X}_{\text{b}}$ has different special latent representations. \footnote{The situation we consider here is the worst case. If we have more different special latent representations in one batch, it would be more difficult to make $\sum_{i=1}^{B}\sum_{j=1}^{B}\mathcal{L}_{i,j}^{(r1)}$ = 0. Therefore, we consider the worst case for training, where all the data in this batch have different special latent representations.}
If \ourtech~is trained close to the optimal solution, according to Equation~\ref{equation15}, for any other data sample whose special latent representation is $\bm{x}_{\text{other}}^{\text{lat\_S}}$ and label is $y_{\text{other}}$, it should satisfy the following conditions established by the data in $\mathbb{X}_{\text{b}}$,
\begin{equation}
\label{equation12}
\begin{cases}
\langle \bm{x}_{\text{other}}^{\text{lat\_S}}, \bm{x}_i^{\text{lat\_S}} \rangle = 1, \text{if} \, y_i = y_{\text{other}} \\
\langle \bm{x}_{\text{other}}^{\text{lat\_S}}, \bm{x}_i^{\text{lat\_S}} \rangle  = 0 , \text{if} \, y_i \neq y_{\text{other}}.
\end{cases}
\end{equation}
Apparently, the resulted coefficient matrix of Equation \ref{equation12} is the special latent representations of $\mathbb{X}_{\text{b}}$, denoted as $\mathbb{X}^{\text{lat\_S}}$. 
Let $\mathbb{X}_q^{\text{lat\_S}}$ be the subset of $\mathbb{X}^{\text{lat\_S}}$, which consists of special latent representations with label value $\bm{a}_q$.

In essence, the dimension of latent representations is the length of variables in the equation system, which can be determined by the relationship between the rank of coefficient matrix and the number of independent solutions. 
Propositions~\ref{lemma1} and \ref{lemma2} give the lower bound of the rank of a coefficient sub-matrix and the lower bound of the rank of the whole coefficient matrix, respectively. 
Based on which, we decide the dimension requirement of latent representations in Theorem~\ref{theorem2}. 

\begin{proposition}
\label{lemma1}
If the dimension $D$ is sufficiently large, then $r(\mathbb{X}_q^{\text{lat\_S}}) \geq N_{q}^{(\text{b})} - 1$, where $r(\mathbb{X}_q^{\text{lat\_S}})$ is the rank of 
$\mathbb{X}_q^{\text{lat\_S}}$.
\end{proposition}
\begin{proof}
Since $D$ is sufficiently large, the main idea of the proof is to construct a group of linearly independent vectors from $\mathbb{X}_{q}^{\text{lat\_S}}$, and the number of vectors in the constructed group is $N_{q}^{(\text{b})} - 1$. 
We can construct the group by randomly selecting a special latent representation $\bm{x}^{\text{lat\_S}}_i$ from $\mathbb{X}_q^{\text{lat\_S}}$ and generating the constructed group by $\mathbb{X}_q^{\text{lat\_S}} - \bm{x}^{\text{lat\_S}}_i$, 
which has $N_{q}^{(\text{b})} - 1$ 
non-zero 
vectors. 
Note that the non-zero vectors in the constructed group are pairwise orthogonal, because for any two non-zero vectors $\bm{x}^{\text{lat\_S}}_{j} - \bm{x}^{\text{lat\_S}}_{i}$ and $\bm{x}^{\text{lat\_S}}_{h} - \bm{x}^{\text{lat\_S}}_{i}$ in the constructed group satisfy $\langle \bm{x}^{\text{lat\_S}}_{j} - \bm{x}^{\text{lat\_S}}_{i},\bm{x}^{\text{lat\_S}}_{h} - \bm{x}^{\text{lat\_S}}_{i} \rangle$ = 0. 
Hence, these vectors are linear independent. 
\end{proof}

\begin{proposition}
\label{lemma2}
If the dimension $D$ is sufficiently large, then $r(\mathbb{X}^{\text{lat\_S}}) \geq B - Q$.
\end{proposition}
\begin{proof}
It is apparent that any two vectors from different constructed groups are orthogonal. 
Therefore, these constructed groups form a large group with $\sum_{q=1}^{Q}N_{q}^{(\text{b})} - Q = B - Q$ independent vectors.
\end{proof}

\begin{theorem}
\label{theorem2}
The dimension of latent representations needs to satisfy: $D \geq B + K - 2Q$.
\end{theorem}
\begin{proof}
According to Proposition~\ref{lemma2}, there exist at most $D - B + Q$ numbers of independent solutions to Equation (\ref{equation15}). 
Moreover, we have $Q \leq K$, and thus $K - Q$ label values are not included in $\mathbb{X}_{\text{b}}$, which means 
there must exist at least $K - Q$ independent solutions to Equation (\ref{equation15}).
Therefore, we can obtain: $D - B + Q \geq K - Q$.
\end{proof}

\subsection{Convergence Analysis}\label{subsec:analysis}

Recall that any design of $\mathcal{L}_{i,j}^{(\text{r1})}$ and $\mathcal{L}_{i}^{(\text{r2})}$ is proposed to satisfy the aforementioned two requirements, and thus the global optimality essentially satisfies the two requirements, as stated in \cref{theorem3}.
\begin{theorem}
\label{theorem3}
$\mathbb{X}^{\text{lat\_S}^*}\left(\mathbb{X}^{\text{lat\_S}^*}\right)^{T} = \bm{S}$, where $\mathbb{X}^{\text{lat\_S}^*}$ represents the optimal solution of special latent representations $\mathbb{X}^{\text{lat\_S}}$.
\end{theorem}
%%% zj: theorem 3.4 is not clear: it satisfies the similarity measurement (zero-one matrix), but how is it related to the second requirements, i.e., reconstruction loss?
%%% zj: for the regularization loss, it is a soft constraint, if \eta is not large enough, the similarity constraint is not necessarily satisfied
%%% xcc: Yes. Therefore, in the experiments, we conduct an ablation study to evaluate the influence of regularization loss. 

Based on the global optimality, we could then analyze the convergence of \ourtech{} training. 
Assume that the original data representation space is discrete, we can then enumerate all the data points in such space. 
Accordingly, we can obtain their special latent representations, i.e., $\mathbb{X}_{W}^{\text{lat\_S}}$.
%%% zj: please check the above
Thus, we have the following theorem.

\begin{theorem}
\label{theorem4}
There always exists a well-trained \ourtech{} such that $\mathbb{X}_{W}^{\text{lat\_S}^*}\left(\mathbb{X}_{W}^{\text{lat\_S}^*}\right)^{T} = \bm{S}_W$, where $\bm{S}_W$ is the similarity among the enumerated data points and $\mathbb{X}_{W}^{\text{lat\_S}^*}$ is the optimal solution of $\mathbb{X}_{W}^{\text{lat\_S}}$.
\end{theorem}
\begin{proof}
We can transform the problem into: for any $\bm{S}_W$, there exists a matrix $\bm{X}$ that satisfies $\bm{X}(\bm{X})^T = \bm{S_{W}}$. The reason is that $\bm{S_{\bm{W}}}$ is a real symmetric matrix, such $\bm{X}$ will always exist. Consequently, if \ourtech~has enough model capacity, we can finally find such a latent space.
\end{proof}
%%% (done) zj: I think this theorem can be removed if there is not enough space. First, real symmetric matrix is the product of a matrix and its transpose is straightforward. Second, "if SGN has enough model capacity, we can finally find such a latent space" --> this sentence does not have much insight, and you do not provide more theoretical results on what is "enough model capacity"?
%%% xcc: I am not sure. May be we could remove theorem 3.4. L_i^{(r2)} is related to the model capacity. If the model has enough capacity, then the network can learn the transformation between the latent space and the original space.
\subsection{Data Generation}
\label{subsec:sgn-generation}

In this section, we present how to generate a data point according to the trained \ourtech{} and a representative dataset $\mathbb{X}_{\text{Rep}}$ in the target domain.
%\textcolor{blue}{which is slightly different in the image and tabular datasets.} 
%%% (done) zj: can you reserve the difference of image and tabular data to the exp part instead of in the technical part
Algorithm \ref{alg1} summarizes \textcolor{black}{the} data generation process.
%%% (done) zj: change to "Algorithm \ref{alg1} summarizes the data generation process."? --> no need to explicitly emphasize "general"
%%% xcc: Thanks for the comment. Algorithm \ref{alg1} only presents a rough generation process. 
%\textcolor{blue}{in the image dataset}. 
First, we decide the label value $\bm{a}_k$ of the generated sample. 
%%% (done) zj: the (step 1), (step 2), (step 3) can be deleted? since it is indicated in Algorithm 1
%%% xcc: I am not sure whether (step 1)... make the statement clearer.
Subsequently, we select a set of $N_k$ data points from $\mathbb{X}_{\text{Rep}}$ whose labels equal $\bm{a}_k$, and obtain $\mathbb{X}_{\text{equal}} = \{(\bm{x}_i,y_i = \bm{a}_k)|i \in \{1,...,N_{k}\}\}$; 
similarly, we select another $N_k$ data points whose labels are not $\bm{a}_k$ to construct $\mathbb{X}_{\text{unequal}} = \{(\bm{x}_i,y_i \neq \bm{a}_k)|i \in \{1,...,N_{k}\}\}$. 
Let $\mathbb{X}_{\text{select}} = \mathbb{X}_{\text{equal}} \cup \mathbb{X}_{\text{unequal}}$ be the selected dataset. 
Next, we construct a latent representation $\bm{x}_{\text{new}}^{\text{lat}}$ that satisfies Equation (\ref{equation15}) for the selected dataset $\mathbb{X}_{\text{select}}$. 
Notice that there exist multiple ways to construct $\bm{x}_{\text{new}}^{\text{lat}}$, i.e., solving the equation system Equation (\ref{equation15}). 
In the experiments, we utilize a simple but effective way that linearly mixes the data points in $\mathbb{X}_{\text{equal}}$. 
%%% zj: add one more sentence to explain the reason here, mix data points or the latent representation? --> this part is not clear
%%%jq: how to construct the latent representation and why does this satisfy Equation 1?
After that, we input $\bm{x}_{\text{new}}^{\text{lat}}$ to the trained \ourtech{}, and the output is the generated data point $\bm{x}_{\text{new}}$. 
Finally, we can obtain a generated data point $(\bm{x}_{\text{new}}, y_{\text{new}})$, where $y_{\text{new}} = \bm{a}_k$. 
%%%jq: how to input the $\bm{x}_{\text{new}}^{\text{lat}}$ to the trained SGN? Need to add a description, since the input of the SGN is the data sample in the training phase, while the input of the SGN is the latent representation of the sample in the generation phase.

%\textcolor{blue}{As for tabular datasets, we utilize a much simpler way. To be specific, we randomly select two data from $\mathbb{X}_{\text{Rep}}$, and then mix up the latent representations of two selected data. Finally, we input the mixed up latent representation to the decoder and reach the generated data $\bm{x}_{\text{new}}$. It can be seen that the mixed latent representation also roughly satasifies Equation (\ref{equation15}).}

A noteworthy aspect is that our data generation method can address the distribution shift problem. 
Specifically, the way of generating data samples should meet the needed distribution of \textcolor{black}{target domain}, which is challenging because of the existence of the distribution shift.
Typically, the data distribution $p(x)$ is determined by the label distribution $p(y)$ and the label-conditioned distribution $p(x|y)$, as shown in the Equation (\ref{equation9}), where
different $p(y)$ (i.e., prior probability shift) and $p(x|y)$
(i.e., covariate shift) could lead to different $p(x)$ (i.e., distribution shift). 
\begin{equation}
\label{equation9}
p(x) = \sum_{k=1}^{K}p(x|y={\bm{a}}_k)p(y={\bm{a}}_k).
\end{equation}
%, which lead to various distributions.}

\begin{algorithm}[t] 
\caption{Data generation based on the trained \ourtech{} model} 
\label{alg1} 
\begin{algorithmic}
\STATE{\textbf{Input:} 
The representative dataset $\mathbb{X}_{\text{Rep}}$, the trained \ourtech{}, the number of data $N_k$ to be sampled.}
\STATE{\textbf{Output:} Generated data: $\left(\bm{x}_{\text{new}},y_{\text{new}}\right)$}.
\STATE{Step\,1:\,Select a label value ${\bm{a}}_k$.}
\STATE{Step\,2:\,Select datasets $\mathbb{X}_{\text{equal}} = \{(\bm{x}_i,y_i = \bm{a}_k)|i \in \{1,...,N_{k}\}\}$ and $\mathbb{X}_{\text{unequal}} = \{(\bm{x}_i,y_i \textcolor{black}{\neq \bm{a}_k})|i \in \{1,...,N_{k}\}\}$ from $\mathbb{X}_{\text{Rep}}$, and set $\mathbb{X}_{\text{select}} = \mathbb{X}_{\text{equal}} \cup \mathbb{X}_{\text{unequal}}$.}
\STATE{Step\,3:\,Compute a latent representation $\bm{x}_{\text{new}}^{\text{lat}}$ %of the generated data 
that satisfies Equation (\ref{equation15}) for $\mathbb{X}_{\text{select}}$.}
\STATE{Step\,4:\,Input 
$\bm{x}_{\text{new}}^{\text{lat}}$ to the trained \ourtech, and the output is the generated data $\bm{x}_{\text{new}}$.}
\STATE{Step\,5:\,Set ${y}_{\text{new}} = \bm{a}_k$ and return $\left(\bm{x}_{\text{new}},y_{\text{new}}\right)$.}
\end{algorithmic}
\end{algorithm}

In \ourtech{}, since we search for a latent space instead of learning the distribution during the training stage, we can use the needed distribution in the target domain to generate the data samples. 
In particular, in Step 1 and Step 2 of Algorithm~\ref{alg1}, we can sample the label value $\bm{a}_k$ and the $\mathbb{X}_{\text{equal}}$ dataset according to $p(y)$ and $p(x|y=\bm{a}_k)$, respectively, where $p(y)$ and $p(x|y=\bm{a}_k)$ are determined by the representative dataset $\mathbb{X}_{\text{Rep}}$. This way, we can tackle the prior probability shift and covariate shift to mitigate the distribution shift problem. %\textcolor{blue}{As for tabular datasets, since the selected two data also satisfies $p(x)$, the generated data is more likely a sample from $p(x)$.}
%In other words, the data latent representations are invariant across different domains. 
Note that in the training stage, the neural network learns the reflection between the data original representations and latent representations. When moving into the target domain, we use the learned reflection to calculate the latent representations of the data used in the target domain. These latent representations inherently consist of the distribution information of the target domain. With the mixture of the latent representations, we can make the generated data satisfy the target domain distribution.

\vspace{-2mm}
\section{Experiments}\label{sec:experiment}

\ignore{
%%% wyc: I comment out the header to save space
In this section, we evaluate the performance of \ourtech{}, where
Section~\ref{subsec:setup} describes the experimental setup, and Section~\ref{subsec:evaluation1} and Section~\ref{subsec:evaluation2} present the experiments for verifying the effectiveness of \ourtech{} on image and tabular datasets, respectively.
}

\vspace{-2mm}
\subsection{Experimental Setup}\label{subsec:setup}

\revise{
We evaluate \ourtech{} on both image datasets and tabular datasets. 
For image datasets, we use two implementations: one is \oursubteca~whose $f(\cdot)$ is an identity function, and the other is \oursubtecb~whose $f(\cdot)$ is the normalized function. 
For the implementation on tabular datasets, we employ $f(\cdot)$ as the normalized function simply denoted \ourtech{}. All the experiments are run on a server equipped with an I9-11900K CPU and GEFORCE RTX 3090 Ti HOF OC LAB Edition * 2 GPUs.
%%% zj: can you explain why choose different functions for image and tabular datasets, it should be related to the data characteristics
%% wyc: it may be better to justify why select these two functions for tabular and image datasets respectively in the appendix, and refer to it from here.
%%% xcc: Thanks for the comments. Actually, these two functions are the most intuitive choices. However, I am not sure whether this can be the reason of selecting these two functions.
%%% wyc: I think it is not a good reason for the justification. As mentioned before, it may be better to explain it in Appendix if the reason is complex; otherwise, the function selection is somewhat ad-hoc. 
%which satisfies Equation \ref{equation7}, and the other is \oursubtecb~which satisfies Equation \ref{equation70}. As for the implementation on tabular datasets, we only present one implementation simply denoted as \ourtech~which satisfies Equation \ref{equation70}. %Moreover, in tabular data generation, we make a slightly change: the linearly mixed up two data aren't  required to have the same label, which is a roughly approximation to Algorithm \ref{alg1}.
}

\revise{
\textbf{Datasets.} We use six image datasets and five tabular datasets for the evaluation. For image datasets, we use BloodMNIST~\citep{ACEVEDO2020105474}, OctMNIST~\citep{kermany2018identifying}, TissueMNIST~\citep{woloshuk2021situ}, RetinaMNIST~\citep{RetiDeep}, SVHN~\citep{svhn} and FashionMNIST~\citep{abs-1708-07747}.
%where the BloodMNIST and RetinaMNIST images are reshaped into ($28 \times 28 \times 3$) based on the pipeline provided by MedMNIST~\citep{yang2021medmnist2}. 
%
We create different domains via the dataset rotation \citep{nguyen2021domain}.
%, which is practical in realistic scenarios. 
%For example, in medical data analysis, it is common to have different domains due to various scanners~\citep{guan2021domain}. 
%
%%% wyc: please add references for these datasets.
For tabular datasets, we utilize Insect~\citep{DBLP:journals/datamine/SouzaRMB20}, Adult~\citep{DBLP:conf/kdd/Kohavi96}, Avila~\citep{DBLP:conf/iciap/StefanoFMF11}, Dry Bean~\citep{DBLP:journals/cea/KokluO20}, and Electrical Grid Stability~\citep{DBLP:conf/smartgridcomm/ArzamasovBJ18} datasets. \textcolor{black}{For the Insect dataset, we utilize sub-datasets that already contain distribution shifts, namely Insects$\_$Abr (denoted InsectsA), Insects$\_$Incr (denoted InsectsI), Insects$\_$IncrGrd (denoted InsectsIG).}
%%%jq: The Insect dataset contains more than four sub-dataset.
%it already contains four distribution-shifted sub-datasets, namely Insects$\_$Abr (denoted InsectsA), Insects$\_$Incr (denoted InsectsI), Insects$\_$IncrGrd (denoted InsectsIG) and Insects$\_$IncrRecr (denoted InsectsIR), 
%and we use two of them each time in the experiments. 
%%% (done) zj: we use two of them each time in the experiments --> please check
For the other four datasets, we divide each of them into two subsets to create distribution shifts. 
The details are explained in Appendix A.
%%% wyc: may put some of the dataset descriptions in Appendix

%we either use existing datasets which already contain distribution shifts or manually create datasets with different distributions by dataset division. To be specific, our experiments are first executed on the Insects dataset \citep{DBLP:journals/datamine/SouzaRMB20} which consists of four sub-datasets: Insects$\_$Abr (denoted InsectsA), Insects$\_$Incr (denoted InsectsI), Insects$\_$IncrGrd (denoted InsectsIG) and Insects$\_$IncrRecr (denoted InsectsIR). All these four datasets embrace different distributions, and we use either two of them to build distribution shift datasets. Furthermore, our experiments are also implemented on Adult, Avila, Dry Bean and Electrical Grid Stability Simulated Datasets. For these four datasets, we divide each one of them into two subsets to create distribution shifts. The specific division details are listed in the supplementary material.

\textbf{Baselines.} For image datasets, we compare \ourtech~with four state-of-the-art CGMs, namely, ProjGAN \citep{miyato2018cgans}, ContraGAN \citep{kang2020contragan}, ReACGAN \citep{kang2021rebooting}, and ADCGAN~\citep{pmlr-v162-hou22a}. 
We adopt the studioGAN library to implement these baseline models. 
The ProjGAN's architecture is based on ResNet, and other baselines' architectures are based on bigGAN. 
For tabular datasets, we compare \ourtech{} with the two conditional tabular dataset generation methods: CTGAN and TVAE \citep{DBLP:conf/nips/XuSCV19}, where we adopt the SDV library to implement them with the defaulted structures stated in the library.
%We utilize the hinge version adversarial loss for all the baseline methods. Regarding the hyper-parameters, we chose the best group of hyperparameters illustrated in each baseline's original paper. We set the number of discriminator update per generator update as 5 and the total generator iteration update as 20,000. 
%All the methods follow the same training and generation pipeline, where the model is first trained in a source domain and then used to generate data in the target domain. 
%Notably, the source and target domains are created by the same dataset because the source and target domains should share similar data but with different distributions. 
}

\textbf{Metric.} 
\textcolor{black}{For image datasets,} we measure the effectiveness of generated data in the target domain by the downstream classification tasks. 
The generated data is considered effective if it can improve the classifier's performance.
\ignore{
In the downstream tasks, 
%i.e., training the classifier in the target domain, 
we manually reduce the number of rotated training data such that the impact of the generated data can be amplified, i.e., for better evaluating the generated data. 
Specifically, the target domain is created from the source domain by rotating the dataset, which consists of training, validation, and testing datasets. 
The rotated validation and testing datasets are still used as validation and testing datasets in the target domain, but we only keep 200 samples per label in the rotated training dataset. 
In this way, the influence of generated data is enlarged when it is included in training the classifier of the downstream task.
%%% (done) zj: this sentence is not clear: the sentence above mention validation and testing are used in the target domain, but this sentence mentions training data 
%%%xcc: Thanks for the comment. We said influence is enlarged because we only keep 200 samples per label in the rotated training dataset.
}
For tabular datasets, we measure the generated data quality by the benchmark provided by \citep{DBLP:conf/nips/XuSCV19}.
%, which measures machine learning efficacy. 
We train regressor/classifier models on the generated data and test the models on the target domain. 
\ignore{
The rationale behind such design is that if the regressor/classifier model works well in the target domain, the regressor/classifier model suits the distribution of the target domain, i.e., the generated data satisfies the target domain distribution. 
We generate 100 samples to train each classifier and regressor. 
}
%We adopt two types of regressors and two types of classifiers. 
%For regressors, the one is
We use Linear Regression (LR) and Multilayer Perceptron (MLP) regression as regressors, and Decision Tree (DT) and MLP classification as classifiers. 
%. For classifiers, the one is Decision Tree, and the other one is MLP classification. 
We utilize the F1 score and $R^2$ to evaluate the classifiers and regressors, respectively. 
%%% wyc: I am not sure if R^2 is a well-known metric, would it be better to give the definition of R^2 here or in appendix
%%% xcc: Thanks for the comment. I think R^2 is a widely used metric for regression models. Given the papers I read, I think maybe we don't need to give the definition of R^2. 
More details on the experimental setup are provided in Appendix B.

%\textbf{Environment.} 

\begin{table*}[b]
\caption{Classification accuracy}
%\scriptsize
%\scriptsize
  \centering
  \label{table:classification}
  \setlength\tabcolsep{2pt}
  \begin{tabular}{lcccccc}
  %\hline 
  \toprule[1pt]
  {Methods} & BloodMNIST & OctMNIST & TissueMNIST & RetinaMNIST &  SVHN  &  FashionMNIST
  \\
  %\hline
  \midrule[0.5pt]
  Original & 0.63 $\pm$ 0.01 & 0.33 $\pm$ 0.03 & 0.38 $\pm$ 0.03 & 0.44 $\pm$ 0.03 & 0.47 $\pm$ 0.05 & 0.41 $\pm$ 0.03
  \\
  ProjGAN & 0.67 $\pm$ 0.02 & 0.32 $\pm$ 0.02 & 0.35 $\pm$ 0.03 & 0.40 $\pm$ 0.04 & 0.49 $\pm$ 0.02 & 0.42 $\pm$ 0.01
  \\ 
  ContraGAN & 0.68 $\pm$ 0.02 & 0.27 $\pm$ 0.03 &  0.39 $\pm$ 0.02 & 0.41 $\pm$ 0.04 & 0.45 $\pm$ 0.01 & 0.40 $\pm$ 0.02 
  \\
  ReACGAN & 0.65 $\pm$ 0.03 & 0.31 $\pm$ 0.03 & 0.40 $\pm$ 0.02 & 0.42 $\pm$ 0.03 &0.52 $\pm$ 0.03 & 0.45 $\pm$ 0.03 
  \\
  ADCGAN & 0.68 $\pm$ 0.02 & 0.34 $\pm$ 0.04 & 0.42 $\pm$ 0.03 & 0.44 $\pm$ 0.03 & 0.51 $\pm$ 0.02 & 0.43 $\pm$ 0.03
  \\
\oursubtecb~ & 0.68 $\pm$ 0.01 & 0.41 $\pm$ 0.03 & 0.42 $\pm$ 0.02 & 0.47 $\pm$ 0.03 & 0.52 $\pm$ 0.02 & 0.49 $\pm$ 0.03 
  \\
  \oursubteca~ & \textbf{0.75 $\pm$ 0.01}  & \textbf{0.43 $\pm$ 0.03} & \textbf{0.44 $\pm$ 0.03} & \textbf{0.50 $\pm$ 0.03} & \textbf{0.55 $\pm$ 0.01} & \textbf{0.53 $\pm$ 0.02} 
  \\
  %\hline
  \bottomrule[1pt] 
  \end{tabular}
\end{table*}

\ignore{
\begin{table*}[b]
\caption{Classification accuracy}
%\scriptsize
%\scriptsize
  \centering
  \setlength\tabcolsep{2pt}
  \label{table:classification}
  \begin{tabular}{lccccccc}
  %\hline 
  \toprule[1pt]
  {Datasets} & Original & \oursubteca~ & \oursubtecb~ & ProjGAN &  ContraGAN & ReACGAN & ADCGAN
  \\
  %\hline
  \midrule[0.5pt]
  BloodMNIST & 0.63 $\pm$ 0.010 & \textbf{0.75 $\pm$ 0.010} & 0.68 $\pm$ 0.013  & 0.67 $\pm$ 0.016 & 0.68 $\pm$ 0.022 & 0.65 $\pm$ 0.031 & \textcolor{black}{0.68 $\pm$ 0.024}
  \\
  OctMNIST & 0.33 $\pm$ 0.029 & \textbf{0.43 $\pm$ 0.025} & 0.41 $\pm$ 0.030 & 0.32 $\pm$ 0.024 & 0.27 $\pm$ 0.033 & 0.31 $\pm$ 0.028 & 0.34 $\pm$ 0.035
  \\
  TissueMNIST & 0.38 $\pm$ 0.025 & \textbf{0.44 $\pm$ 0.026} & 0.42 $\pm$ 0.021 & 0.35 $\pm$ 0.027 & 0.39 $\pm$ 0.023 & 0.40 $\pm$ 0.024 & 0.42 $\pm$ 0.033
  \\
  RetinaMNIST & 0.44 $\pm$ 0.032 & \textbf{0.50 $\pm$ 0.029} & 0.47 $\pm$ 0.031 & 0.40 $\pm$ 0.039 & 0.41 $\pm$ 0.037 & 0.42 $\pm$ 0.032 & 0.44 $\pm$ 0.031
  \\
  SVHN & 0.47 $\pm$ 0.045 & \textbf{0.55 $\pm$ 0.009} &0.52 $\pm$ 0.019 & 0.49 $\pm$ 0.018 & 0.45 $\pm$ 0.014 & 0.52 $\pm$ 0.028 & 0.51 $\pm$ 0.015
  \\
  FashionMNIST & 0.41 $\pm$ 0.026 & \textbf{0.53 $\pm$ 0.015} & 0.49 $\pm$ 0.031 & 0.42 $\pm$ 0.014 & 0.40 $\pm$ 0.019 & 0.45 $\pm$ 0.032 & 0.43 $\pm$ 0.027
  \\
  %\hline
  \bottomrule[1pt] 
  \end{tabular}
\end{table*}
}

\begin{figure*}[t]
	\centering
    \includegraphics[scale=0.33]{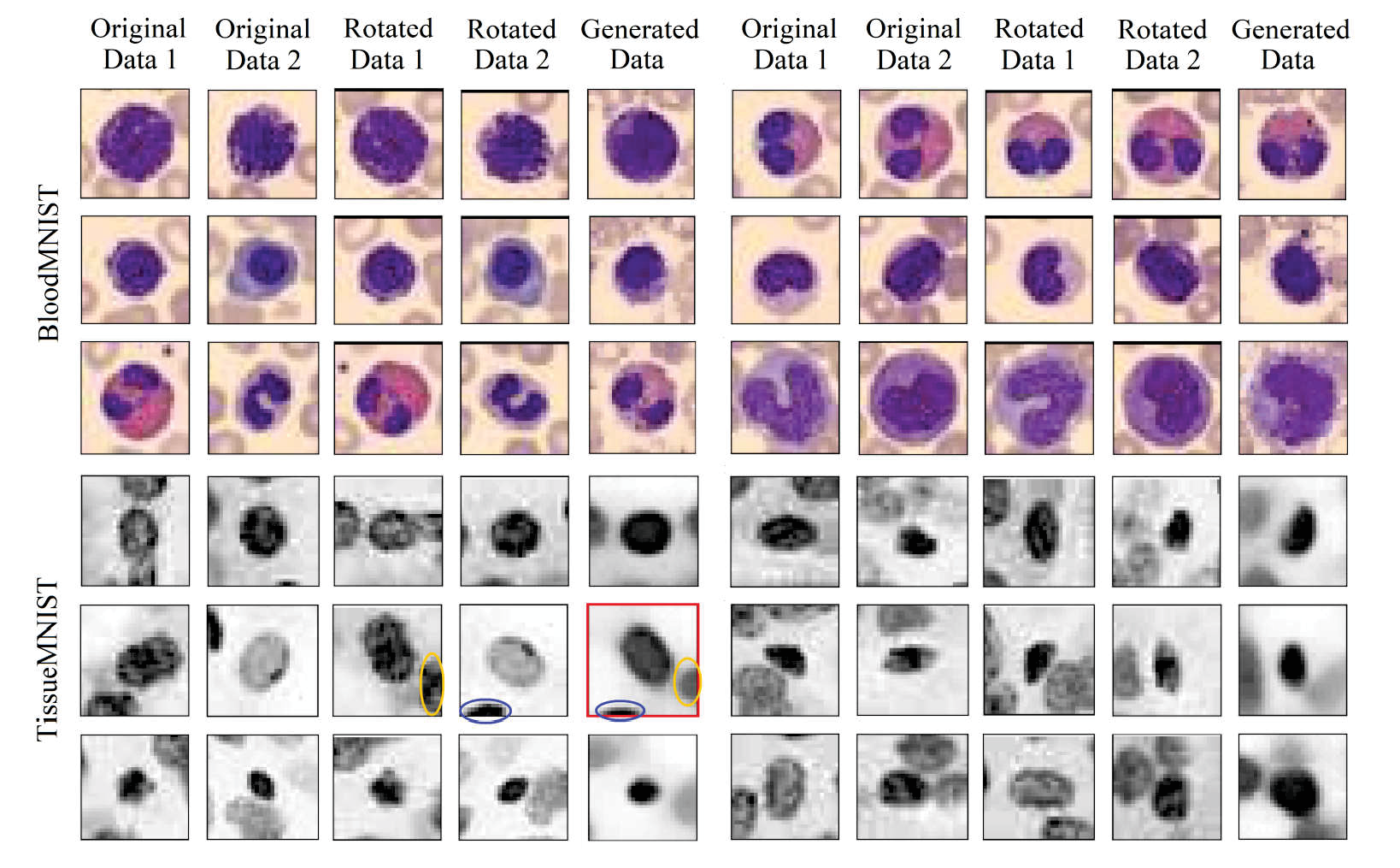}
    \vspace{-1mm}
    \caption{\revise{Visualization results of the RetinaMNIST and TissueMNIST datasets.} The Original Data 1 and Original Data 2 are in the source domain for training the \ourtech. The Rotated Data 1 and Rotated Data 2 are in the target domain.
    %, obtained by rotating the corresponding original data 90 degrees.
    %The Generated Data is the linear combination of Rotated Data 1 and Rotated Data 2 obtained in the data generation stage.
    \revise{The Generated Data is obtained in the data generation stage.}
    } 
	\vspace{-6mm}
    \label{fig2}
\end{figure*}

\vspace{-5mm}
\subsection{Evaluation on Image Datasets} \label{subsec:evaluation1}

%%% wyc: would it be better to move some of the visualization results back? the current experiments only contain three tables.
%%% (done) zj: reduce some theorem and move some images here
\textbf{Downstream task evaluation.}
In this paper, we conduct classification as the downstream tasks, and the classifier is ResNet50~\citep{DBLP:conf/cvpr/HeZRS16}. 
The target domain is created by rotating the original dataset for 90 degrees. 
For each generative paradigm, we generate 200 images per label, the same number as that in the rotated training dataset. 
Table~\ref{table:classification} summarizes the classification results  over three runs evaluated on five generative paradigms. 
Note that the `Original' column refers to the classification accuracy directly on the rotated training dataset (without generated data).
The accuracy is low due to the limited number of training images.

There are three main observations from Table~\ref{table:classification}. 
First, the new data generated by the baselines, i.e., ProjGAN, ContraGAN, RcACGAN, and ADCGAN, is ineffective for downstream tasks in the target domain. 
Specifically, except for the BloodMNIST dataset, the performances of the classifiers trained with generated data on the other five datasets are similar to or even lower than those trained without generated data. 
This is because the baselines are trained in the source domain, i.e., they learn the source domain's distributions for generating new data. 
When the distribution shifts exist between the source domain and target domain, the generated new data is not helpful to the classification tasks in the target domain. 
For the BloodMNIST dataset, the \textcolor{black}{four} baselines can improve the classifiers' performances because the key features of the label information are more robust. 
%%% zj: can explain more about the "robust"?
Note that BloodMNIST results in the highest `Original' classification accuracy; thus, data rotation has little impact on the key features. 
Second, the data generated by \ourtech{} greatly improves the classification accuracy of the downstream tasks w.r.t. the \textcolor{black}{six} datasets. 
Even for the BloodMNIST dataset, where all the generative paradigms can improve the accuracy, \ourtech{} still outperforms the baselines. 
The reason is that \ourtech{} essentially learns a latent space during training and incorporates the target domain's distribution into the data generation process. 
%%% zj: target domain's distribution: do you mean the label information?
Therefore, it can generate data that follows the target domain's distribution, leading to better performance. Third, the performance of the \ourtech{}'s \oursubteca~realization is generally better than that of the \oursubtecb~realization. 
For example, \oursubteca~achieves about $2\% \sim 6 \%$ accuracy improvement compared to \oursubtecb. The reason is that \oursubteca~employs more orthogonal bases to construct the latent representations, and thus, it has a higher capability to provide high-quality latent representations for generating better data. 
%Next, we evaluate a downstream semi-supervised classification task w.r.t. the BloodMNIST dataset, where we further generate 100 images without labels beside the 200 generated labeled images. Then, we adopt a semi-supervised classifier for the training. We observe that NeurSpec (i.e., 0.74 $\pm$ 0.021) still achieves the best performances compared with all baselines, i.e., 0.69 $\pm$ 0.018, 0.65 $\pm$ 0.021, 0.64 $\pm$ 0.021, 0.66 $\pm$ 0.052, for ProjGAN, ContraGAN, ReACCGAN, ADCGAN, respectively. 

\textbf{Visualization results.}
Figure~\ref{fig2} shows the visualization results of data generation in \oursubteca~as it provides better accuracy on the downstream classification tasks. % compared to \oursubtecb. 
Figure~\ref{fig2} consists of three parts: original data, rotated data, and generated data, where the original data is used for \oursubteca~training, the rotated data is the input of generation, and the generated data is the output. 
Moreover, the original data is in the source domain, while the rotated and generated data are in the target domain. 
We can observe from Figure~\ref{fig2} that, even though \oursubteca~is not trained in the target domain, the generated data is similar to but different from the rotated data, demonstrating that \oursubteca~is able to generate acceptable data by the linear combination of rotated data. 
For example, the third generated image of TissueMNIST, i.e., the fifth image in the fifth row highlighted in the red rectangle, consists of two circles, where the yellow circle inherits from rotated 1, and the blue circle inherits from rotated data 2. 
%%% (done) zj: please check the above
%
\textcolor{black}{More visualization results of the generated data are provided in Appendix C.}
%%% (done) zj: what is the specific section number of the appendix?
%%% wyc: better to link to specific Appenaaaion.

%\textbf{Further evaluation on distribution shift problems.}
%Finally, we construct a more challenging target domain of the BloodMNIST dataset to further evaluate the effectiveness of \ourtech{} on distribution shift problems. Specifically, we adopt an angle-hybrid target domain, i.e., we split each data sample into four pieces, and rotate them by 0, 90, 180, and 270 degrees, respectively. The experimental results demonstrate that our NeurSpec method still achieves the highest accuracy, i.e., 0.73 $\pm$ 0.20, confirming its superior performance. While the accuracies of the ProjGAN, ContraGAN, and ReACCGAN baselines are: 0.65 $\pm$ 0.015, 0.66 $\pm$ 0.012, and 0.64 $\pm$ 0.017, respectively. The main reason is that different rotated degrees have little impact on the zero-one similarity matrix and thus have little impact on the special latent presentation space. Therefore, NeurSpec can generate useful data to improve the classifier's accuracy. 

\begin{table*}[t]
\caption{Benchmark results ($R^2$ for regressors) on tabular datasets}
\vspace{-1mm}
%\scriptsize
  \centering
  \setlength\tabcolsep{4.5pt}
  \label{table:regression}
  \begin{tabular}{llcccccc}
  %\hline 
  \toprule[1pt]
  \multirow{2}{*}{\textcolor{black}{Dataset}} & \multirow{2}{*}{\textcolor{black}{Source $\rightarrow$ Target}} & \multicolumn{2}{c}{\ourtech{}}  & \multicolumn{2}{c}{CTGAN} & \multicolumn{2}{c}{TVAE} 
  \\
  & & LR & MLP & LR & MLP 
  & LR & MLP 
  \\
  %\hline
  \midrule[0.5pt]
  \multirow{12}{*}{\textcolor{black}{Insect}} & InsectA $\to$ InsectI & \textbf{0.5028} & 0.0422 & -0.8304 & -0.2169 & 0.4923 & -0.2887
  \\
  & InsectA $\to$ InsectIG & \textbf{0.4620} & -0.1132 & -1.3399 & -0.4884 & -0.0043 & -0.9879 
  \\
  & InsectA $\to$ InsectIR & \textbf{0.4040} & -0.0600 & -0.8327 & -0.2939 & 0.2641 & -0.2887
  \\
  & InsectI $\to$ InsectA & \textbf{0.5055} & -0.0525 & -0.0950 & -0.4944 & 0.4964 & -0.0928
  \\
  & InsectI $\to$ InsectIG & \textbf{0.5108} & -0.0984 & -0.9219 & -0.9173 & 0.2754 & -0.1894
  \\
  & InsectI $\to$ InsectIR & \textbf{0.5153} & -0.1894 & -0.2775 & -0.5936 & 0.4581 & -0.1164
  \\
  & InsectIG $\to$ InsectA & \textbf{0.5723} & -0.0837 & -0.3835 & -0.5262 & 0.2108 & -0.3124
  \\
  & InsectIG $\to$ InsectI & \textbf{0.2833} & -0.1305 & -0.4712 & -0.4354 & 0.2261 & -0.2591
  \\
  & InsectIG $\to$ InsectIR & \textbf{0.3754} & 0.0116 & -0.4423 & -0.4035 & 0.1576 & -0.0977
  \\
  & InsectIR $\to$ InsectA & \textbf{0.5229} & -0.0694 & 0.2952 & -0.4751 & 0.2813 & -0.0267
  \\
  & InsectIR $\to$ InsectI & \textbf{0.6332} & 0.0477 & 0.3664 & -0.4710 & 0.2457 & 0.0253
  \\
  & InsectIR $\to$ InsectIG & \textbf{0.4970} & -0.0905 & 0.3193 &  -0.4689 & 0.1291 & -0.0010
  \\
  %\hline
  \midrule[0.5pt]
  \multirow{2}{*}{\textcolor{black}{Electrial}} & Electrical0 $\to$ Electrical1 & \textbf{0.0987} & -0.4460 & -0.0518 & -1.4251 & -0.6375 & -8.0482
  \\
  & Electrical1 $\to$ Electrical0 & \textbf{-0.1539} & -1.5070 & -1.1412 & -2.4765 & -1.8348 & -7.2222 \\
  %\hline
  \bottomrule[1pt] 
  \end{tabular}
\end{table*}

\begin{table}[t]
\caption{Benchmark results (F1 score for classifiers) on tabular datasets}
%\scriptsize
  \centering
  \label{table:classifier}
  \begin{tabular}{llcccccc}
  %\hline 
  \toprule[1pt]
   \multirow{2}{*}{\textcolor{black}{Dataset}} & \multirow{2}{*}{\textcolor{black}{Source $\rightarrow$ Target}} & \multicolumn{2}{c}{\ourtech{}}  & \multicolumn{2}{c}{CTGAN} & \multicolumn{2}{c}{TVAE} 
  \\
  & & DT & MLP & DT & MLP 
  & DT & MLP 
  \\
  %\hline
  \midrule[0.5pt]
  \multirow{2}{*}{\textcolor{black}{Adult}} & Adult0 $\to$ Adult1 & 0.5560 & \textbf{0.5723} & 0.3888 & 0.5707 & 0.4291 & 0.5711
  \\
  & Adult1 $\to$ Adult0 & 0.6834 & \textbf{0.8810} & 0.4348 & 0.8754 & 0.2409 & 0.1181 
  \\
  %\hline
  \midrule[0.5pt]
  \multirow{2}{*}{\textcolor{black}{DryBean}} & DryBean0 $\to$ DryBean1 & 0.0717 & \textbf{0.2306} & 0.0563 & 0.0759 & 0.0699 & 0.0704
  \\
  & DryBean1 $\to$ DryBean0 & 0.2198 & \textbf{0.2313} & 0.1717 & 0.2156 & 0.0606 & 0.1487
  \\
  %\hline
  \midrule[0.5pt]
  \multirow{2}{*}{\textcolor{black}{Avila}} & Avila0 $\to$ Avila1 & \textbf{0.3016} & 0.2975 & 0.0692 & 0.1706 & 0.2714 & 0.2621
  \\
  & Avila1 $\to$ Avila0 & \textbf{0.4625} & 0.4323 & 0.0918 & 0.0713 & 0.3459 & 0.3844
  \\
  %\hline
  \bottomrule[1pt]  
  \end{tabular}
\end{table}

\vspace{-2mm}
\subsection{\textcolor{black}{Evaluation on Tabular Datasets}}\label{subsec:evaluation2}

\revise{
%\textbf{Comparison with baselines.}
Tables \ref{table:regression} and \ref{table:classifier} summarize the comparison results for regressors and classifiers, respectively.
%The comparison results are presented in , where Table \ref{table:regression} consists of $R^2$ of regressors and Table \ref{table:classifier} is about the F1 score of classifiers. 
%We can reach following three conclusions. 
There are two observations. 
First, \ourtech{} outperforms CTGAN and TVAE in most cases. This is because \ourtech{} is designed to learn a universal linear space and transmission across the target and source domains. However, CTGAN and TVAE are designed to learn the distributions of the source domain, and thus, cannot perform well in the target domain.
\ignore{
Second, linear regressor usually performs better than MLP regressor, while MLP classifier often performs better than decision tree classifiers. This demonstrates the meaning of using two types of classifiers and regressors which avoids the possible error caused by the single type of regressor/classifier. 
}
%%% wyc: what is the point of the second observation? do you want to mention that two regressors/classifiers can make the evaluation more robust? This seems not to be a good observation. Would it be possible to explain why linear regressors outperform MLP regressors, and similar for the classifiers?
%%% xcc: Maybe we can just simply remove the second observation?
%%% wyc: can, if no good explanation, but this subsection is a bit short
%%% xcc: Got it, thanks.
Second, we observe that the results are not symmetric, meaning that the performance of \ourtech{} from one dataset to another may not be the same in reverse. For example, \ourtech{} performs well from Avila0 to Avia1, %which does not mean \ourtech{} would have the same F1 
%%% ooibc: f1--> F1
%score 
but it does not perform well from Avila1 to Avila0. This lack of symmetry can be attributed to differences between the datasets.
While GMs trained on one dataset may find it easier to perform well on another, this does not necessarily guarantee the same for models trained on a different dataset.
%
%%% wyc: I suggest putting at least some of the ablation study here instead of in the appendix.
We further conduct an ablation study of the effect of the parameter $\eta$ in the loss function in Appendix D.
}

\vspace{-2mm}
\section{Conclusion}

In this paper, we propose a similarity-based generative network \ourtech{} to generate desirable samples for downstream tasks. 
%%% (done) zj: a new data generative framework --> framework or mechanism? I think you use mechanism throughout the paper, please be consistent
%%%xcc: Thanks for the comment. I think maybe using 'scheme' to describe SGN and 'mechanism' to describe the similarity mechanism is better?
The core of \ourtech{} is a novel similarity-based evaluation mechanism, which utilizes the data labels to construct a zero-one similarity matrix. 
In the training stage, instead of learning the probability distribution of the training dataset in the source domain, \ourtech{} searches for a latent space that satisfies the similarity requirement. 
In the data generation stage, \ourtech{} incorporates the needed distribution information of the representative dataset to generate the desirable samples in the target domain, mitigating the distribution shift problem. 
Extensive experiments on various real-world datasets, \textcolor{black}{including both image and tabular datasets}, demonstrate the effectiveness of \ourtech{} and its superiority over the baselines.
%%% (done) zj: the page limit is nine pages for contents
%%%xcc: Got it. I will move some contents to the supplementary material.

%%%%%%%%%%%%%%%%%%%%%%%%%%%%%%%%%%%%%%%%%%%%%%%%%%%%%%%%%%%%

\bibliographystyle{ACM-Reference-Format}
\bibliography{Ref}

@inproceedings{DBLP:conf/kdd/Kohavi96,
  author       = {Ron Kohavi},
  title        = {Scaling Up the Accuracy of Naive-Bayes Classifiers: {A} Decision-Tree
                  Hybrid},
  booktitle    = {SIGKDD},
  pages        = {202--207},
  year         = {1996}
}

@inproceedings{DBLP:conf/smartgridcomm/ArzamasovBJ18,
  author       = {Vadim Arzamasov and
                  Klemens B{\"{o}}hm and
                  Patrick Jochem},
  title        = {Towards Concise Models of Grid Stability},
  booktitle    = {SmartGridComm},
  pages        = {1--6},
  year         = {2018}
}

@article{DBLP:journals/cea/KokluO20,
  author       = {Murat Koklu and
                  Ilker Ali {\"{O}}zkan},
  title        = {Multiclass classification of dry beans using computer vision and machine
                  learning techniques},
  journal      = {Comput. Electron. Agric.},
  volume       = {174},
  pages        = {105507},
  year         = {2020}
}

@inproceedings{DBLP:conf/iciap/StefanoFMF11,
  author       = {Claudio De Stefano and
                  Francesco Fontanella and
                  Marilena Maniaci and
                  Alessandra Scotto di Freca},
  title        = {A Method for Scribe Distinction in Medieval Manuscripts Using Page
                  Layout Features},
  booktitle    = {{ICIAP}},
  volume       = {6978},
  pages        = {393--402},
  year         = {2011}
}

@article{DBLP:journals/datamine/SouzaRMB20,
  author       = {Vin{\'{\i}}cius M. A. de Souza and
                  Denis Moreira dos Reis and
                  Andr{\'{e}} Gustavo Maletzke and
                  Gustavo E. A. P. A. Batista},
  title        = {Challenges in benchmarking stream learning algorithms with real-world
                  data},
  journal      = {Data Min. Knowl. Discov.},
  volume       = {34},
  number       = {6},
  pages        = {1805--1858},
  year         = {2020}
}

@inproceedings{DBLP:conf/nips/XuSCV19,
  author       = {Lei Xu and
                  Maria Skoularidou and
                  Alfredo Cuesta{-}Infante and
                  Kalyan Veeramachaneni},
  title        = {Modeling Tabular data using Conditional {GAN}},
  booktitle    = {NeurIPS},
  pages        = {7333--7343},
  year         = {2019}
}

@InProceedings{pmlr-v162-hou22a,
  title = 	 {Conditional {GAN}s with Auxiliary Discriminative Classifier},
  author =       {Hou, Liang and Cao, Qi and Shen, Huawei and Pan, Siyuan and Li, Xiaoshuang and Cheng, Xueqi},
  booktitle = 	 {ICML},
  pages = 	 {8888--8902},
  year = 	 {2022},
  volume = 	 {162},
  month = 	 {17--23 Jul},
}

@article{ACEVEDO2020105474,
title = {A dataset of microscopic peripheral blood cell images for development of automatic recognition systems},
journal = {Data in Brief},
volume = {30},
pages = {105474},
year = {2020},
author = {Andrea Acevedo and Anna Merino and Santiago Alférez and et al.}
}

@inproceedings{shaham2019singan,
  title={Singan: Learning a generative model from a single natural image},
  author={Shaham, Tamar Rott and Dekel, Tali and Michaeli, Tomer},
  booktitle={ICCV},
  pages={4570--4580},
  year={2019}
}

@inproceedings{yi2017dualgan,
  title={Dualgan: Unsupervised dual learning for image-to-image translation},
  author={Yi, Zili and Zhang, Hao and Tan, Ping and Gong, Minglun},
  booktitle={ICCV},
  pages={2849--2857},
  year={2017}
}

@article{deldjoo2021survey,
  title={A survey on adversarial recommender systems: from attack/defense strategies to generative adversarial networks},
  author={Deldjoo, Yashar and Noia, Tommaso Di and Merra, Felice Antonio},
  journal={ACM Computing Surveys},
  volume={54},
  number={2},
  pages={1--38},
  year={2021},
}

@article{NIPS2014_5ca3e9b1,
  title={Generative adversarial nets},
  author={Goodfellow, Ian and Pouget-Abadie, Jean and Mirza, Mehdi and Xu, Bing and Warde-Farley, David and Ozair, Sherjil and Courville, Aaron and Bengio, Yoshua},
  journal={NIPS},
  volume={27},
  year={2014}
}

@inproceedings{DBLP:journals/corr/KingmaW13,
  author    = {Diederik P. Kingma and
               Max Welling},
  title     = {Auto-Encoding Variational Bayes},
  booktitle = {{ICLR}},
  year      = {2014},
}

@article{mirza2014conditional,
  title={Conditional generative adversarial nets},
  author={Mirza, Mehdi and Osindero, Simon},
  journal={arXiv preprint arXiv:1411.1784},
  year={2014}
}

@article{sohn2015learning,
  title={Learning structured output representation using deep conditional generative models},
  author={Sohn, Kihyuk and Lee, Honglak and Yan, Xinchen},
  journal={NIPS},
  volume={28},
  pages={3483--3491},
  year={2015}
}

@book{quinonero2008dataset,
  title={Dataset shift in machine learning},
  author={Qui{\~n}onero-Candela, Joaquin and Sugiyama, Masashi and Schwaighofer, Anton and Lawrence, Neil D},
  year={2008},
}

@article{kermany2018identifying,
  title={Identifying medical diagnoses and treatable diseases by image-based deep learning},
  author={Kermany, Daniel S and Goldbaum, Michael and Cai, Wenjia and Valentim, Carolina CS and Liang, Huiying and Baxter, Sally L and McKeown, Alex and Yang, Ge and Wu, Xiaokang and Yan, Fangbing and others},
  journal={Cell},
  volume={172},
  number={5},
  pages={1122--1131},
  year={2018},
}

@article{woloshuk2021situ,
  title={In situ classification of cell types in human kidney tissue using 3D nuclear staining},
  author={Woloshuk, Andre and Khochare, Suraj and Almulhim, Aljohara F and McNutt, Andrew T and Dean, Dawson and Barwinska, Daria and Ferkowicz, Michael J and Eadon, Michael T and Kelly, Katherine J and Dunn, Kenneth W and others},
  journal={Cytometry Part A},
  volume={99},
  number={7},
  pages={707--721},
  year={2021},
}

@misc{RetiDeep,
    author = "{DeepDRiD}",
    title = {The 2nd diabetic retinopathy – grading and image quality estimation challenge},
    year = {2020},
}

@article{kang2020contragan,
  title={Contragan: Contrastive learning for conditional image generation},
  author={Kang, Minguk and Park, Jaesik},
  journal={NeurIPS},
  year={2020}
}

@article{kang2021rebooting,
  title={Rebooting ACGAN: Auxiliary Classifier GANs with Stable Training},
  author={Kang, Minguk and Shim, Woohyeon and Cho, Minsu and Park, Jaesik},
  journal={NeurIPS},
  volume={34},
  year={2021}
}

@inproceedings{miyato2018cgans,
  author    = {Takeru Miyato and
               Masanori Koyama},
  title     = {cGANs with Projection Discriminator},
  booktitle = {{ICLR}},
  year      = {2018}
}

@article{nguyen2021domain,
  title={Domain invariant representation learning with domain density transformations},
  author={Nguyen, A Tuan and Tran, Toan and Gal, Yarin and Baydin, Atilim Gunes},
  journal={NeurIPS},
  volume={34},
  year={2021}
}

@inproceedings{DBLP:conf/cvpr/HeZRS16,
  author    = {Kaiming He and
               Xiangyu Zhang and
               Shaoqing Ren and
               Jian Sun},
  title     = {Deep Residual Learning for Image Recognition},
  booktitle = {{CVPR}},
  pages     = {770--778},
  year      = {2016}
}

@article{abs-1708-07747,
  author       = {Han Xiao and
                  Kashif Rasul and
                  Roland Vollgraf},
  title        = {Fashion-MNIST: a Novel Image Dataset for Benchmarking Machine Learning
                  Algorithms},
  journal      = {CoRR},
  volume       = {abs/1708.07747},
  year         = {2017}
}

@inproceedings{svhn,
title	= {Reading Digits in Natural Images with Unsupervised Feature Learning},
author	= {Yuval Netzer and Tao Wang and Adam Coates and Alessandro Bissacco and Bo Wu and Andrew Y. Ng},
year	= {2011}
}

@article{DBLP:journals/csur/YangZSHXZZCY24,
  author       = {Ling Yang and
                  Zhilong Zhang and
                  Yang Song and
                  Shenda Hong and
                  Runsheng Xu and
                  Yue Zhao and
                  Wentao Zhang and
                  Bin Cui and
                  Ming{-}Hsuan Yang},
  title        = {Diffusion Models: {A} Comprehensive Survey of Methods and Applications},
  journal      = {{ACM} Comput. Surv.},
  volume       = {56},
  number       = {4},
  pages        = {105:1--105:39},
  year         = {2024},
  url          = {https://doi.org/10.1145/3626235},
  doi          = {10.1145/3626235},
  bibsource    = {dblp computer science bibliography, https://dblp.org}
}

@inproceedings{DBLP:conf/nips/HoJA20,
  author       = {Jonathan Ho and
                  Ajay Jain and
                  Pieter Abbeel},
  editor       = {Hugo Larochelle and
                  Marc'Aurelio Ranzato and
                  Raia Hadsell and
                  Maria{-}Florina Balcan and
                  Hsuan{-}Tien Lin},
  title        = {Denoising Diffusion Probabilistic Models},
  booktitle    = {Advances in Neural Information Processing Systems 33: Annual Conference
                  on Neural Information Processing Systems 2020, NeurIPS 2020, December
                  6-12, 2020, virtual},
  year         = {2020},
  url          = {https://proceedings.neurips.cc/paper/2020/hash/4c5bcfec8584af0d967f1ab10179ca4b-Abstract.html},
  bibsource    = {dblp computer science bibliography, https://dblp.org}
}

@inproceedings{DBLP:conf/iccv/ChouBH23,
  author       = {Gene Chou and
                  Yuval Bahat and
                  Felix Heide},
  title        = {Diffusion-SDF: Conditional Generative Modeling of Signed Distance
                  Functions},
  booktitle    = {{IEEE/CVF} International Conference on Computer Vision, {ICCV} 2023,
                  Paris, France, October 1-6, 2023},
  pages        = {2262--2272},
  publisher    = {{IEEE}},
  year         = {2023},
  url          = {https://doi.org/10.1109/ICCV51070.2023.00215},
  doi          = {10.1109/ICCV51070.2023.00215},
  bibsource    = {dblp computer science bibliography, https://dblp.org}
}

@article{zhu2026generative,
  title={Generative anomaly detection: a comprehensive review of modeling principles, advances, and future opportunities},
  author={Zhu, Jiaqi and Fan, Yunfeng and Han, Geng and Shi, Xiang and Deng, Fang and Chen, Jie},
  journal={Artificial Intelligence Review},
  year={2026},
  publisher={Springer}
}

@article{zhu2023meter,
  title={METER: A Dynamic Concept Adaptation Framework for Online Anomaly Detection},
  author={Zhu, Jiaqi and Cai, Shaofeng and Deng, Fang and Ooi, Beng Chin and Zhang, Wenqiao},
  journal={Proceedings of the VLDB Endowment},
  volume={17},
  number={4},
  pages={794--807},
  year={2023},
  publisher={VLDB Endowment}
}

@inproceedings{zhu2025context,
  title={In-Context Adaptation to Concept Drift for Learned Database Operations},
  author={Zhu, Jiaqi and Cai, Shaofeng and Shen, Yanyan and Chen, Gang and Deng, Fang and Ooi, Beng Chin},
  booktitle={International Conference on Machine Learning},
  pages={79699--79726},
  year={2025},
  organization={PMLR}
}

@inproceedings{DBLP:conf/icml/HouCSPLC22,
  author       = {Liang Hou and
                  Qi Cao and
                  Huawei Shen and
                  Siyuan Pan and
                  Xiaoshuang Li and
                  Xueqi Cheng},
  editor       = {Kamalika Chaudhuri and
                  Stefanie Jegelka and
                  Le Song and
                  Csaba Szepesv{\'{a}}ri and
                  Gang Niu and
                  Sivan Sabato},
  title        = {Conditional GANs with Auxiliary Discriminative Classifier},
  booktitle    = {International Conference on Machine Learning, {ICML} 2022, 17-23 July
                  2022, Baltimore, Maryland, {USA}},
  series       = {Proceedings of Machine Learning Research},
  volume       = {162},
  pages        = {8888--8902},
  publisher    = {{PMLR}},
  year         = {2022},
  url          = {https://proceedings.mlr.press/v162/hou22a.html},
  bibsource    = {dblp computer science bibliography, https://dblp.org}
}

@inproceedings{DBLP:conf/iccv/PeeblesX23,
  author       = {William Peebles and
                  Saining Xie},
  title        = {Scalable Diffusion Models with Transformers},
  booktitle    = {{IEEE/CVF} International Conference on Computer Vision, {ICCV} 2023,
                  Paris, France, October 1-6, 2023},
  pages        = {4172--4182},
  publisher    = {{IEEE}},
  year         = {2023},
  url          = {https://doi.org/10.1109/ICCV51070.2023.00387},
  doi          = {10.1109/ICCV51070.2023.00387},
  bibsource    = {dblp computer science bibliography, https://dblp.org}
}

@inproceedings{DBLP:conf/iclr/LiHCLS0025,
  author       = {Xirui Li and
                  Charles Herrmann and
                  Kelvin C. K. Chan and
                  Yinxiao Li and
                  Deqing Sun and
                  Chao Ma and
                  Ming{-}Hsuan Yang},
  title        = {A Simple Approach to Unifying Diffusion-based Conditional Generation},
  booktitle    = {The Thirteenth International Conference on Learning Representations,
                  {ICLR} 2025, Singapore, April 24-28, 2025},
  publisher    = {OpenReview.net},
  year         = {2025},
  url          = {https://openreview.net/forum?id=tAGmxz1TUi},
  bibsource    = {dblp computer science bibliography, https://dblp.org}
}

@inproceedings{DBLP:conf/iclr/SadatKHW25,
  author       = {Seyedmorteza Sadat and
                  Manuel Kansy and
                  Otmar Hilliges and
                  Romann M. Weber},
  title        = {No Training, No Problem: Rethinking Classifier-Free Guidance for Diffusion
                  Models},
  booktitle    = {The Thirteenth International Conference on Learning Representations,
                  {ICLR} 2025, Singapore, April 24-28, 2025},
  publisher    = {OpenReview.net},
  year         = {2025},
  url          = {https://openreview.net/forum?id=b3CzCCCILJ},
  bibsource    = {dblp computer science bibliography, https://dblp.org}
}

@inproceedings{DBLP:conf/cvpr/WangGBHK020,
  author       = {Yaxing Wang and
                  Abel Gonzalez{-}Garcia and
                  David Berga and
                  Luis Herranz and
                  Fahad Shahbaz Khan and
                  Joost van de Weijer},
  title        = {MineGAN: Effective Knowledge Transfer From GANs to Target Domains
                  With Few Images},
  booktitle    = {2020 {IEEE/CVF} Conference on Computer Vision and Pattern Recognition,
                  {CVPR} 2020, Seattle, WA, USA, June 13-19, 2020},
  pages        = {9329--9338},
  publisher    = {Computer Vision Foundation / {IEEE}},
  year         = {2020},
  url          = {https://openaccess.thecvf.com/content\_CVPR\_2020/html/Wang\_MineGAN\_Effective\_Knowledge\_Transfer\_From\_GANs\_to\_Target\_Domains\_With\_CVPR\_2020\_paper.html},
  doi          = {10.1109/CVPR42600.2020.00935},
  bibsource    = {dblp computer science bibliography, https://dblp.org}
}

@inproceedings{nguyen2026targeted,
  title     = {Do We Need All the Synthetic Data? Targeted Image Augmentation via Diffusion Models},
  author    = {Nguyen, Dang and Li, Jiping and Zheng, Jinghao and Mirzasoleiman, Baharan},
  booktitle = {International Conference on Learning Representations},
  year      = {2026}
}

@inproceedings{DBLP:conf/icml/VermaLBNMLB19,
  author       = {Vikas Verma and
                  Alex Lamb and
                  Christopher Beckham and
                  Amir Najafi and
                  Ioannis Mitliagkas and
                  David Lopez{-}Paz and
                  Yoshua Bengio},
  editor       = {Kamalika Chaudhuri and
                  Ruslan Salakhutdinov},
  title        = {Manifold Mixup: Better Representations by Interpolating Hidden States},
  booktitle    = {Proceedings of the 36th International Conference on Machine Learning,
                  {ICML} 2019, 9-15 June 2019, Long Beach, California, {USA}},
  series       = {Proceedings of Machine Learning Research},
  volume       = {97},
  pages        = {6438--6447},
  publisher    = {{PMLR}},
  year         = {2019},
  url          = {http://proceedings.mlr.press/v97/verma19a.html},
  bibsource    = {dblp computer science bibliography, https://dblp.org}
}

@inproceedings{DBLP:conf/iclr/ZhangCDL18,
  author       = {Hongyi Zhang and
                  Moustapha Ciss{\'{e}} and
                  Yann N. Dauphin and
                  David Lopez{-}Paz},
  title        = {mixup: Beyond Empirical Risk Minimization},
  booktitle    = {6th International Conference on Learning Representations, {ICLR} 2018,
                  Vancouver, BC, Canada, April 30 - May 3, 2018, Conference Track Proceedings},
  publisher    = {OpenReview.net},
  year         = {2018},
  url          = {https://openreview.net/forum?id=r1Ddp1-Rb},
  bibsource    = {dblp computer science bibliography, https://dblp.org}
}

@article{DBLP:journals/corr/abs-2511-18281,
  author       = {Yara Bahram and
                  Melodie Desbos and
                  Mohammadhadi Shateri and
                  Eric Granger},
  title        = {Uni-DAD: Unified Distillation and Adaptation of Diffusion Models for
                  Few-step Few-shot Image Generation},
  journal      = {CoRR},
  volume       = {abs/2511.18281},
  year         = {2025},
  url          = {https://doi.org/10.48550/arXiv.2511.18281},
  doi          = {10.48550/ARXIV.2511.18281},
  eprinttype   = {arXiv},
  eprint       = {2511.18281},
  bibsource    = {dblp computer science bibliography, https://dblp.org}
}

@article{DBLP:journals/corr/abs-2207-12598,
  author       = {Jonathan Ho and
                  Tim Salimans},
  title        = {Classifier-Free Diffusion Guidance},
  journal      = {CoRR},
  volume       = {abs/2207.12598},
  year         = {2022},
  url          = {https://doi.org/10.48550/arXiv.2207.12598},
  doi          = {10.48550/ARXIV.2207.12598},
  eprinttype   = {arXiv},
  eprint       = {2207.12598},
  bibsource    = {dblp computer science bibliography, https://dblp.org}
}

@inproceedings{DBLP:conf/nips/DhariwalN21,
  author       = {Prafulla Dhariwal and
                  Alexander Quinn Nichol},
  editor       = {Marc'Aurelio Ranzato and
                  Alina Beygelzimer and
                  Yann N. Dauphin and
                  Percy Liang and
                  Jennifer Wortman Vaughan},
  title        = {Diffusion Models Beat GANs on Image Synthesis},
  booktitle    = {Advances in Neural Information Processing Systems 34: Annual Conference
                  on Neural Information Processing Systems 2021, NeurIPS 2021, December
                  6-14, 2021, virtual},
  pages        = {8780--8794},
  year         = {2021},
  url          = {https://proceedings.neurips.cc/paper/2021/hash/49ad23d1ec9fa4bd8d77d02681df5cfa-Abstract.html},
  bibsource    = {dblp computer science bibliography, https://dblp.org}
}

@article{DBLP:journals/corr/abs-2211-03264,
  author       = {Jingyuan Zhu and
                  Huimin Ma and
                  Jiansheng Chen and
                  Jian Yuan},
  title        = {Few-shot Image Generation with Diffusion Models},
  journal      = {CoRR},
  volume       = {abs/2211.03264},
  year         = {2022},
  url          = {https://doi.org/10.48550/arXiv.2211.03264},
  doi          = {10.48550/ARXIV.2211.03264},
  eprinttype   = {arXiv},
  eprint       = {2211.03264},
  bibsource    = {dblp computer science bibliography, https://dblp.org}
}

@inproceedings{DBLP:conf/iclr/HuZYHWM25,
  author       = {Teng Hu and
                  Jiangning Zhang and
                  Ran Yi and
                  Hongrui Huang and
                  Yabiao Wang and
                  Lizhuang Ma},
  title        = {SaRA: High-Efficient Diffusion Model Fine-tuning with Progressive
                  Sparse Low-Rank Adaptation},
  booktitle    = {The Thirteenth International Conference on Learning Representations,
                  {ICLR} 2025, Singapore, April 24-28, 2025},
  publisher    = {OpenReview.net},
  year         = {2025},
  url          = {https://openreview.net/forum?id=wGVOxplEbf},
  bibsource    = {dblp computer science bibliography, https://dblp.org}
}

@inproceedings{DBLP:conf/iclr/ZhongZ0L25,
  author       = {Jincheng Zhong and
                  Xiangcheng Zhang and
                  Jianmin Wang and
                  Mingsheng Long},
  title        = {Domain Guidance: {A} Simple Transfer Approach for a Pre-trained Diffusion
                  Model},
  booktitle    = {The Thirteenth International Conference on Learning Representations,
                  {ICLR} 2025, Singapore, April 24-28, 2025},
  publisher    = {OpenReview.net},
  year         = {2025},
  url          = {https://openreview.net/forum?id=PplM2kDrl3},
  bibsource    = {dblp computer science bibliography, https://dblp.org}
}

@inproceedings{DBLP:conf/iclr/TrabuccoDGS24,
  author       = {Brandon Trabucco and
                  Kyle Doherty and
                  Max Gurinas and
                  Ruslan Salakhutdinov},
  title        = {Effective Data Augmentation With Diffusion Models},
  booktitle    = {The Twelfth International Conference on Learning Representations,
                  {ICLR} 2024, Vienna, Austria, May 7-11, 2024},
  publisher    = {OpenReview.net},
  year         = {2024},
  url          = {https://openreview.net/forum?id=ZWzUA9zeAg},
  bibsource    = {dblp computer science bibliography, https://dblp.org}
}

@inproceedings{DBLP:conf/cvpr/Wang025a,
  author       = {Yanghao Wang and
                  Long Chen},
  title        = {Inversion Circle Interpolation: Diffusion-based Image Augmentation
                  for Data-scarce Classification},
  booktitle    = {{IEEE/CVF} Conference on Computer Vision and Pattern Recognition,
                  {CVPR} 2025, Nashville, TN, USA, June 11-15, 2025},
  pages        = {25560--25569},
  publisher    = {Computer Vision Foundation / {IEEE}},
  year         = {2025},
  url          = {https://openaccess.thecvf.com/content/CVPR2025/html/Wang\_Inversion\_Circle\_Interpolation\_Diffusion-based\_Image\_Augmentation\_for\_Data-scarce\_Classification\_CVPR\_2025\_paper.html},
  doi          = {10.1109/CVPR52734.2025.02380},
  bibsource    = {dblp computer science bibliography, https://dblp.org}
}

\clearpage

\end{document}